\documentclass[runningheads]{llncs}
\usepackage[T1]{fontenc}
\usepackage{graphicx}
\usepackage{capt-of}  
\usepackage{cuted}
\usepackage[backend=biber, sorting=none, giveninits=true]{biblatex}
\usepackage{graphicx}%
\usepackage{multirow}%
\usepackage{amsmath,amssymb,amsfonts}%
\usepackage{mathrsfs}%
\usepackage[figuresright]{rotating}%
\usepackage{xcolor}%
\usepackage{textcomp}%
\usepackage{manyfoot}%
\usepackage{booktabs}%
\usepackage{algorithm}%
\usepackage{algorithmicx}%
\usepackage{algpseudocode}%
\usepackage{program}%
\usepackage{listings}%
\usepackage{wrapfig}
\usepackage{vruler}
\usepackage{setspace}
\usepackage{comment}
\usepackage{siunitx}
\usepackage{booktabs}

\begin{document}
\title{Computational Dualism and Objective Superintelligence}
%
%
\author{Michael Timothy Bennett\inst{1}\\\orcidID{0000-0001-6895-8782} 
}
\authorrunning{Michael Timothy Bennett}
%
\institute{The Australian National University\\
\email{michael.bennett@anu.edu.au}
\\ \url{http://www.michaeltimothybennett.com/}
}
\maketitle              

\begin{abstract}
The concept of intelligent software is flawed. The behaviour of software is determined by the hardware that ``interprets'' it. This undermines claims regarding the behaviour of theorised, software superintelligence. 
Here we characterise this problem as ``computational dualism'', where instead of mental and physical substance, we have software and hardware. We argue that to make objective claims regarding performance we must avoid computational dualism. We propose a pancomputational alternative wherein every aspect of the environment is a relation between irreducible states. We formalise systems as behaviour (inputs and outputs), and cognition as embodied, embedded, extended and enactive. The result is cognition formalised as a part of the environment, rather than as a disembodied policy interacting with the environment through an interpreter. This allows us to make objective claims regarding intelligence, which we argue is the ability to ``generalise'', identify causes and adapt. We then establish objective upper bounds for intelligent behaviour. This suggests AGI will be safer, but more limited, than theorised.
\end{abstract}

%

\section{Introduction}
%
%
%
%
\label{intro}
AIXI \cite{hutter2010} is a general reinforcement learning agent. It uses Solomonoff Induction, a formalism of Ockham's Razor \cite{solomonoff1978}, to make accurate inferences from minimal data. It was initially thought to be pareto optimal, representing an upper bound on intelligence \cite{legg2007}. Unfortunately, this claim was later shown to be a matter of interpretation \cite{leike2015}. We argue that this is a valuable insight indicative of a much larger problem with how artificial intelligence (AI) is conceived. We call this problem ``computational dualism'', in reference to René Descarte's interactionist, substance dualism. We then discuss an alternative formulation of intelligent behaviour that permits objective performance claims. It is based upon enactive cognition \cite{thompson2007}, pancomputationalism \cite{piccinini2021} and weak constraint optimisation \cite{bennett2023b}. We use it to propose upper bounds on intelligent behaviour. \\

\noindent \textbf{Intelligence: } Our focus is not AIXI, but intelligence. General intelligence is often defined in terms of predictive models \cite{hutter2010,chollet2019,bennett2022a}. A more ``intelligent'' predictive model \textit{generalises}, making accurate predictions in \textit{unfamiliar} circumstances. It \textit{adapts}. The more efficiently one adapts, the more ``intelligent'' one is. Predictive accuracy is not what matters. With enough training examples even a lookup table can make accurate predictions (because it will have seen every example). What matters is how \textit{efficiently} one learns what \textit{caused} examples. Knowing cause, one can make accurate predictions. Intuitively, a model \textit{explains} the present by identifying those aspects of the past which caused it. 
Using such a model, we might use the more distant past to explain events in the more recent past, and the present to predict the future. It is the future with which we are concerned, as an agent that can accurately predict the results of its actions can choose the actions that yield the most reward. Hence the ability to adapt to unforeseen circumstances, satisfy goals and otherwise behave intelligently can be equated with the ability to identify cause and effect \cite{pearl2018}.
Of course, the only \textit{truyly} accurate model of the environment \textit{is} the environment. Everything else is an abstraction. Hence even if we only consider models that comprehensively explain the past, some will ``generalise'' better than others. \\

\noindent \textbf{Simplicity: } This is where Ockham's Razor comes in. All else being equal, it implies that simpler models are more likely to generalise \cite{sober2015}.
Simplicity is often formalised as Kolmogorov Complexity (KC) \cite{kolmogorov_1963} or minimum description length \cite{rissanen1978}. The KC of an object is the length of the shortest self extracting archive of that object. Intuitively, the KC of the past is the shortest comprehensive description of the past. More compressed descriptions tend to generalise better. This is why some believe that compression and intelligence are closely related \cite{chaitin2006}.
Formally, in the case of AIXI, if the model which generated past data is indeed computable, then the simplest model will dominate the Bayesian posterior as more and more data is observed. Eventually, AIXI will have identified the correct model, which it can use to generate the next sample (predict the future). \\

\noindent \textbf{Subjectivity: } We will now use AIXI to illustrate the problem with software ``intelligence''. KC (and thus AIXI's performance) is measured in the context of a UTM. By itself, changing the UTM would not meaningfully affect performance. When used in a universal prior to predict deterministic binary sequences, the number of incorrect predictions a model will make is bounded by a multiple of the KC of that model \cite{solomonoff1978}. If the UTM is changed the number of errors only changes by a constant  \cite[2.1.1 \& 3.1.1]{vitanyi2008}, so changing the UTM doesn't change which model is considered most plausible. However, when AIXI employs this prior in an \textit{interactive} setting, a problem occurs \cite{leike2015}. Intuitively (with significant abuse of notation), assume a program $f_1$ is software, $f_2$ is an interpreter and $f_3$ is the reality (an environment, body etc) within which goals are pursued.
According to Ockham's Razor, AIXI is the optimal choice of $f_1$ to maximise the performance of $f_3(f_2(f_1))$. However, in an interactive setting one's perception of success may not match reality.
\begin{quote}
``Legg-Hutter intelligence \cite{legg2007} is measured with respect to a fixed UTM. AIXI is the most intelligent policy if it uses the same UTM.'' \cite[p.10]{leike2015}
\end{quote}
\noindent Using our informal analogy of functions, this means performance in terms of $f_3(f_2(f_1))$ depends upon $f_2(f_1)$, not $f_1$ alone. A claim regarding the performance of $f_1$ alone would be \textit{subjective}, in that it depends upon $f_2$. 
\begin{quote}
``This undermines all existing optimality properties for AIXI.'' \cite[p.1]{leike2015}
\end{quote}

\noindent \textbf{Computational dualism: } We suggest this problem has broader significance for AI. The concept of a software ``mind'' interacting with a hardware ``body'' echoes Descartes' interactionist, substance dualism \cite{bennett2022c}. Descartes argued mental and physical ``substances'' interact through the pineal gland which interprets mental events to cause physical events \cite{kim2011}, like a UTM interprets software. When later scholars pointed out the inconsistencies implied by this interaction between mental and physical, some argued that the mental ``supervenes'' on the physical, meaning any two objects that are exactly the same mentally must be exactly the same physically. More recently, philosophers have proposed purely physicalist depictions of the mind \cite{putnam1967}, and have even gone so far as to formalise cognition not just as embodied, but as enacted through the environment \cite{thompson2007}. One might argue that AI is the engineering branch of philosophy of mind and cognitive science. Instead of mental and physical substances, we have software and hardware. Software ``supervenes'' on hardware; any two computers that are exactly alike in hardware configuration down to the value of each and every bit, must be exactly alike in terms of software. However, we have not moved on from dualism and formalised more recent conceptions of mind. AI still tends to be equated with ``immortal'' software. Unless there exists a Platonic realm of pure math (akin to a mental realm), Hinton's ``immortal'' computations \cite{hinton2022} do not exist and never have. There are only embodied ``mortal'' computations, because software is nothing more than the configuration of hardware. Our understanding of AI should be revised to reflect this. This is what we set out to do here.\\ 

\noindent \textbf{Summary: } We propose computational dualism, argue AI warrants an alternative, discuss one such alternative\footnote{Definitions or variations thereof are shared with related work
\cite{bennett2023b,bennett2023c,bennett2023d,bennett2024b,bennett2024c}.}, and then propose an objective upper bound on intelligent behaviour. 

\section{Pancomputational Enactivism}
To make objective claims we must avoid computational dualism. One alternative is enactivism \cite{ward2017}. It holds that mind, body and environment are inseparable. Cognition extends into the environment, and is enacted through what the organism does. For example, if someone uses pen and paper to solve a math problem, then cognition is enacted using the pen and paper \cite{clark1998}. 
To formalise enactivism, we must formalise cognition as a part of the environment. We do so in pancomputationalist \cite{piccinini2021} terms, albeit a very minimal interpretation thereof.
Pancomputationalism holds that everything is a computational system. Using the analogy from earlier, we formalise $f_2(f_3(f_1))$ instead of $f_3(f_2(f_1))$. One may regard the interpreter $f_2$ as the laws of nature, the environment $f_3$ as software running on $f_2$, and we're seeking to portray the mind $f_1$ as subject to the environment. 
Because of this, the distinction between software and hardware must be discarded. Instead, we describe artificial minds in a purely behaviourist manner \cite{bennett2022a}. We describe inputs and outputs rather than a mechanism by which one is mapped to the other. Correct policies are then all possible ``causal intermediaries'' between a set of inputs and outputs, akin to functionalist explanations of human mentality \cite{putnam1967}. \\

\noindent \textbf{The environment: } 
\label{modelwithinmodel}
Rather than assuming programs, we arrive at a pancomputational model of the environment by assuming first that some things exist, some do not, and that the environment is that which exists. Second, we assume there is at least one state of the environment. States could be differentiated along dimensions like time, but we don't strictly need to assume that. We also don't need to assume anything about the internal structure or contents of states. Instead, we can formalise the set of all declarative programs\footnote{Intuitively, declarative programs are anything which is true or false.} as the powerset of states. A declarative program is ``true'' about every state it contains, and false about everything else. Every conceivable environment or state thereof amounts to a set of true declarative programs, or ``facts''.  

\begin{definition}[environment]\label{environment}\hphantom{.}
\begin{itemize}{\small
    \item  We assume a set $\Phi$ whose elements we call \textbf{states}.
    \item A \textbf{declarative program} is $f \subseteq \Phi$, and we write $P$ for the set of all declarative programs (the powerset of $\Phi$).
    \item By a \textbf{truth} or \textbf{fact} about a state $\phi$, we mean $f \in P$ such that $\phi \in f$. 
    \item By an \textbf{aspect of a state} $\phi$ we mean a set $l$ of facts about $\phi$ s.t. $\phi \in \bigcap l$. By an \textbf{aspect of the environment} we mean an aspect $l$ of any state, s.t. $\bigcap l \neq \emptyset$. We say an aspect of the environment is \textbf{realised}\footnote{Realised meaning it is made real, or brought into existence.} by state $\phi$ if it is an aspect of $\phi$. 
    }
\end{itemize}
\end{definition}
Our approach reflects Goertzel's framing of unary, dyadic and triadic relations \cite{goertzel2006}. A state here is unary, referring only to itself. However, a declarative program is dyadic relation between states, in the sense that it refers to states which are its truth conditions. Sets of declarative programs form a lattice based on truth conditions, which relates them to one another. We can then define triadic relations by taking three sets of declarative programs $i,o,\pi$ such that $\pi$ is a ``causal intermediary'' implying $o$ given $i$.\\

\noindent \textbf{Embodiment: } We use these dyadic relations to formalise enactivism \cite{thompson2007}, holding that everything we call the mind is just part of the environment, and ``thinking'' is just changes in environmental state. This blurs the line between agent and environment, making the distinction unclear.  
As Heidegger maintained, Being is bound by context \cite{heidegger2020}. There is no need to define an agent that has no environment, and so there seems to be little point in preserving the distinction.
What we need is not a Cartesian model of disembodied intelligence looking in upon an environment but an embodiment, embedded in a particular part of the environment, through which goal directed behaviour can be enacted.
We need to describe the \textit{circuitry} with which cognition is embodied and enacted, much like a \textbf{formal language} only with truth conditions determined by whatever laws govern the environment. Just as every computational system we can build is finite, and living systems are ergodic \cite{friston2013}, we assume all of this takes place within a system which has a finite number of configurations, and so amounts to a finite number of declarative programs. We call this an \textbf{abstraction layer}. Intuitively, an abstraction layer is like a window. It looks onto that part of the environment in which cognition takes place. For example, we might enumerate every possible declarative program which pertains only to one specific computer, and use that computer as our abstraction layer. However, as we can take any set of declarative programs and define an abstraction layer, it is ``pancomputational'' in the sense that it captures every system in computational terms. 

\begin{definition}[abstraction layer] \label{abstractionlayer}\hphantom{.}
\begin{itemize}{\small 
    \item We single out a subset $\mathfrak{v} \subseteq P$ which we call \textbf{the vocabulary} of an abstraction layer. The vocabulary is finite unless explicitly stated otherwise. If $\mathfrak{v} = P$, then we say that there is no abstraction.
    \item $L_\mathfrak{v} = \{ l \subseteq \mathfrak{v} : \bigcap l \neq \emptyset \}$ is a set of aspects in $\mathfrak{v}$. We call $L_\mathfrak{v}$ a formal language, and $l \in L_\mathfrak{v}$ a \textbf{statement}.
    \item We say a statement is \textbf{true} given a state iff it is an aspect realised by that state.  
    \item A \textbf{completion} of a statement $x$ is a statement $y$ which is a superset of $x$. If $y$ is true, then $x$ is true. 
    \item The \textbf{extension of a statement} $x \in {L_\mathfrak{v}}$ is $E_x = \{y \in {L_\mathfrak{v}} : x \subseteq y\}$. $E_x$ is the set of all completions of $x$. 
    \item The \textbf{extension of a set of statements} $X \subseteq {L_\mathfrak{v}}$ is $E_X = \bigcup\limits_{x \in X} E_x$. 
    \item We say $x$ and $y$ are \textbf{equivalent} iff $E_x = E_y$.
    }
\end{itemize}

\noindent {\normalfont(notation)} $E$ with a subscript is the extension of the subscript\footnote{e.g. $E_l$ is the extension of $l$.}.
\end{definition}
Intuitively, $L_\mathfrak{v}$ is everything which can be realised in this abstraction layer. The extension $E_x$ of a statement $x$ is the set of all statements whose existence implies $x$, and so it is like a truth table. Now that we have an embodiment, we need to define goal directed behaviour. An abstraction layer is \textit{already} goal directed in the sense that it constrains what might be described, just as living organisms evolve to thrive in particular environments. However, an organism can then engage in more \textit{specific} goal directed behaviours by \textit{learning} and adapting to remain fit in more circumstances. This is what we seek to formalise \textit{using} an abstraction layer. It is where a predictive model might fit. However we do not need a value-neutral model of the environment \cite{bennett2022a}. \begin{quote}
“The best model of the world is the world itself.” - Rodney Brooks \cite{dreyfus2007}
\end{quote}The only aspects of the environment that we might actually need model are those necessary to satisfy goals \cite{bennett2022b}. What we need is a model of a task, or more specifically to enumerate the goal directed behaviour (inputs and outputs) that will eventually cause the task to be completed. Goal directed here just means some outputs are ``correct'', and some are not. It is arbitrary.

\begin{definition}[{$\mathfrak{v}$}-task]\label{task} For a chosen $\mathfrak{v}$, a task $\alpha$ is a pair $\langle {I}_\alpha, {O}_\alpha \rangle$ where:\begin{itemize}{ \small
    \item ${I}_\alpha \subset L_\mathfrak{v}$ is a set whose elements we call \textbf{inputs} of $\alpha$. 
    \item ${O_\alpha} \subset E_{I_\alpha}$ is a set whose elements we call \textbf{correct outputs} of $\alpha$. 
} 
\end{itemize}
${I_\alpha}$ has the extension $E_{I_\alpha}$ we call \textbf{outputs}, and ${O_\alpha}$ are outputs deemed correct. 
$\Gamma_\mathfrak{v}$ is the set of \textbf{all tasks} given $\mathfrak{v}$.\\

\noindent {\normalfont(generational hierarchy)} A $\mathfrak{v}$-task $\alpha$ is a \textbf{child} of $\mathfrak{v}$-task $\omega$ if ${I}_\alpha \subset {I}_\omega$ and ${O}_\alpha \subseteq {O}_\omega$. This is written as $\alpha \sqsubset \omega$. If $\alpha \sqsubset \omega$ then $\omega$ is then a \textbf{parent} of $\alpha$. $\sqsubset$ implies a ``lattice'' or generational hierarchy of tasks. 
Formally, the level of a task $\alpha$ in this hierarchy is the largest $k$ such there is a sequence $\langle \alpha_0, \alpha_1, ... \alpha_k \rangle$ of $k$ $\mathfrak{v}$-tasks where $\alpha_0 = \alpha$ and $\alpha_i \sqsubset \alpha_{i+1}$ for all $i \in (0,k)$. A child is always ``lower level'' than its parents. \\

\noindent{\normalfont(notation)} If $\omega \in \Gamma_\mathfrak{v}$, then we will use subscript $\omega$ to signify parts of $\omega$, meaning one should assume $\omega = \langle {I}_\omega, {O}_\omega \rangle$ even if that isn't written. 
\end{definition} 

Intuitively, an \textbf{input} is a possibly incomplete description or task. An \textbf{output} is a completion of an input [def.\ \ref{abstractionlayer}]. We treat \textbf{correctness} as binary. An output is correct if it causes the task to become complete to some acceptable degree with some acceptable probability\footnote{Also called ``satisficing'' a goal \cite{artinger2022}.}. Degrees of complete or correct just reflect different $\mathfrak{v}$-task definitions \footnote{Affect or reward and attribution thererof are beyond this paper's scope.}. 
For example we might put this in evolutionary terms, where the inputs and outputs are the enumeration of all ``fit'' behaviour in which an organism might engage, as remaining alive is goal directed behaviour. We now formalise learning and inference of goal directed behaviour as tasks. 
Inference requires a ``policy''. Being a set of declarative programs, a correct policy \textit{is} the goal of a $\mathfrak{v}$-task, so this amounts to goal learning. 

\begin{definition}[inference]\label{inference}\hphantom{.}
\begin{itemize}{\small
    \item A {$\mathfrak{v}$}-task \textbf{policy} is a statement $\pi \in L_\mathfrak{v}$. It constrains how we complete inputs.
    \item $\pi$ is a \textbf{correct policy} iff the correct outputs $O_\alpha$ of $\alpha$ are exactly the completions $\pi'$ of $\pi$ such that $\pi'$ is also a completion of an input.   
    \item The set of all correct policies for a task $\alpha$ is denoted $\Pi_\alpha$.\footnote{To repeat the definition in set builder notation:
${\Pi_\alpha} = \{\pi \in L_\mathfrak{v} : {E}_{I_\alpha} \cap E_{\pi} = {O_\alpha}\}$}
}\end{itemize}

Assume $\mathfrak{v}$-task $\omega$ and a policy $\pi \in L_\mathfrak{v}$. Inference proceeds as follows:\begin{enumerate}{\small
    \item we are presented with an input ${i} \in {I}_\omega$, and
    \item we must select an output $e \in E_{i} \cap E_\pi$.
    \item If $e \in {O}_\omega$, then $e$ is correct and the task ``complete''. $\pi \in {\Pi}_\omega$ implies $e \in {O}_\omega$, but $e \in {O}_\omega$ doesn't imply $\pi \in {\Pi}_\omega$ (an incorrect policy can imply a correct output).}
\end{enumerate} 
\end{definition}

Intuitively, a \textbf{policy} constrains how we complete inputs. It is a \textbf{correct policy} if it constrains us to correct outputs. To ``learn'' a policy we use a \textbf{proxy}. A proxy estimates one thing, by measuring another seemingly unrelated thing. For example, AIXI uses simplicity to estimate model veracity. In our case, we want a policy that classifies correct outputs. We will use a proxy called ``weakness'', which has been shown to outperform simplicity in sample efficiency and causal learning \cite{bennett2023b,bennett2023c}. Where simplicity is a property of \textit{form}, weakness is a property of \textit{function}. In order to make objective claims regarding performance, we cannot rely upon subjective interpretations of form. 

\begin{definition}[learning]\label{learning} \phantom{.}
\begin{itemize}{\small
    \item A \textbf{proxy} $<$ is a binary relation on statements. $<_w$ is the \textbf{weakness} proxy. For statements $l_1,l_2$ we have  $l_1 <_w l_2$ iff $\lvert E_{l_1}\rvert < \lvert E_{l_2} \rvert$.}
\end{itemize}

\noindent {\normalfont(generalisation)} A statement $l$ \textbf{generalises} to a $\mathfrak{v}$-task $\alpha$ iff $l \in \Pi_\alpha$. We speak of \textbf{learning} $\omega$ from $\alpha$ iff, given a proxy $<$, $\pi \in {\Pi}_\alpha$ maximises $<$ relative to all other policies in ${\Pi}_\alpha$, and $\pi \in {\Pi}_\omega$.\\

\noindent {\normalfont(probability of generalisation)} We assume a uniform distribution over $\Gamma_\mathfrak{v}$. 
If $l_1$ and $l_2$ are policies, we say it is less probable that $l_1$ generalizes than that $l_2$ generalizes, written $l_1 <_g l_2$, iff, when a task $\alpha$ is chosen at random from $\Gamma_\mathfrak{v}$ (using a uniform distribution) then the probability that $l_1$ generalizes to $\alpha$ is less than the probability that $l_2$ generalizes to $\alpha$. \\

\noindent {\normalfont(sample efficiency)} Suppose $\mathfrak{app}$ is the set of \textbf{a}ll \textbf{p}airs of \textbf{p}olicies. Assume a proxy $<$ returns $1$ iff true, else $0$. Proxy $<_a$ is more sample efficient than $<_b$ iff $$\left ( \sum_{(l_1,l_2) \in \mathfrak{app}} \lvert (l_1 <_g l_2) - (l_1 <_a l_2) \rvert - \lvert (l_1 <_g l_2) - (l_1 <_b l_2) \rvert  \right ) < 0$$  

\noindent {\normalfont(optimal proxy)} There is no proxy more sample efficient than $<_w$, so we call $<_w$ optimal. This formalises the idea that    
``explanations should be no more specific than necessary'' (see Bennett's razor in \cite{bennett2023b}).
\end{definition}

\noindent Learning is an activity undertaken by some manner of agent, and a task has been ``learned'' when that agent knows a correct policy. For example, one has ``learned'' chess when one knows the rules and some winning strategies. Here instead of agents, we have embodied goal-directed behaviours in the form of $\mathfrak{v}$-tasks. Humans typically learn from ``examples''. In the context of a $\mathfrak{v}$-task an ``example'' is a correct output and an input that is a subset thereof. A collection of examples is a child task. ``Learning'' is an attempt to generalise from a known child to one of its parents. Intuitively, a child functions like a ``history'' or memory of \textit{correct} interactions (an ``ostensive definition''). The lower the child is in the generational hierarchy one learns from, the more sample efficiently one learns. We assume tasks are \textbf{uniformly distributed} because anything else would imply that environment ``prefers'' or makes a value judgement about goals. The environment \textit{is} a value judgement about what exists, but is otherwise assumed to be impartial. Consequently the most sample efficient proxy is $<_w$ \cite{bennett2023b}. It enables causal learning \cite{bennett2023c} as the weakest correct policies have the highest probability of being the causal intermediary between inputs and outputs. 

\section{Limits of Intelligence}
As stated earlier, intelligence is often understood in terms of ``the ability to generalise'' \cite{chollet2019} and learn the policy which ``caused'' examples \cite{bennett2023c}. We wish to understand the \textit{objective} upper limits of intelligence. $<_w$ is the optimal choice of proxy for sample efficiency \cite[prop. 1, 2, 3]{bennett2023b}. In the context of a fixed abstraction layer, the upper bound of intelligent behaviour is attained by the weakness proxy.
However, more intelligence is not always useful. Intuitively, making a human more intelligent is unlikely to improve their driving. Likewise, intelligence conveys no advantage if one's purpose is to ascribe meaning to random noise. The extent to which it is \textit{possible} to construct weak correct policies depends on both the \textit{abstraction layer} and the task. 
An abstraction layer is a bottleneck on intelligence, and can be goal directed just like a task. 
We can measure this goal directed ``utility'' as the extent to which it is \textit{possible} to construct weaker correct policies given a task in a particular abstraction layer.

\begin{definition}[utility of intelligence]\label{utility} Utility is the difference in weakness between the weakest and strongest correct policies of a task. The utility of a $\mathfrak{v}$-task $\gamma$ is $\epsilon(\gamma) = \underset{{\pi} \in {\Pi}_{\gamma}}{\max} \left(\lvert E_\pi \rvert - \lvert {O}_{\gamma}\rvert \right)$.
\end{definition}
To maximise probability of generalisation, we use the weakness proxy \textit{and} try to maximise utility. Increasing utility changes the abstraction layer to allow construction of weaker correct policies. Beyond a certain point, utility may be increased by increasing $\lvert L_\mathfrak{v} \rvert$ without actually changing the weakest correct policies $\Pi_\gamma$ contains (just the size of their extensions). This means it is not \textit{always} necessary to increase utility to construct policies that generalise, but it is helpful up to a point. Utility is maximised when $\mathfrak{v}=P$, though in practice finite resources would limit us to smaller vocabularies. $\Gamma_P$ contains all tasks in all vocabularies. Hence, for every task $\rho$ in $\Gamma_P$ we can define a function that takes a vocabulary $\mathfrak{v}$ and returns a $\mathfrak{v}$-task which is a child of $\rho$. This lets us represent a task in different vocabularies to compare their utility.
\begin{definition}[uninstantiated-tasks]\label{uninstantiated}
The set of all tasks with no abstraction (meaning $\mathfrak{v}=P$) is $\Gamma_P$ (it contains every task in every vocabulary). For every $P$-task $\rho \in \Gamma_P$ there exists a function $\lambda_\rho : 2^P \rightarrow \Gamma_P$ that takes a vocabulary $\mathfrak{v}' \in 2^P$ and returns a $\mathfrak{v}'$-task $\omega \sqsubset \rho$. We call $\lambda_\rho$ an \textbf{uninstantiated-task}. It is instantiated by choosing a vocabulary. 
\end{definition}

\begin{proposition}[upper bound]\label{proof_bound} 
The most `intelligent' choice of policy and vocabulary given uninstantiated task $\lambda_\rho$ is $\pi$ and $\mathfrak{v}$ s.t. $\mathfrak{v}$ maximises utility for $\lambda_\rho(\mathfrak{v})$, $\pi \in \Pi_{\lambda_\rho(\mathfrak{v})}$ and $\pi$ maximises weakness.
\end{proposition}

\begin{proof} We have equated intelligence with sample efficient generalisation. According to \cite[prop. 1,2]{bennett2023b} the weakest correct policies have the highest probability of generalising. Given an uninstantiated task $\lambda_\rho$, utility measures the weakness of the weakest correct policies. We can use this to compare vocabularies. By choosing a vocabulary $\mathfrak{v}$ which maximises utility for $\lambda_\rho(\mathfrak{v})$, we instantiate $\lambda_\rho$ in a vocabulary that maximises the weakness of correct policies for $\lambda_\rho$. Then, using weakness proxy, we can select a policy that has the highest possible probability of generalising, and thus maximise sample efficiency. \qed \hphantom{} \end{proof}

Put another way, utility is maximised for $\lambda_\rho(\mathfrak{v})$ when $\mathfrak{v}=P$. When utility is maximised there must exist $\pi \in \Pi_{\lambda_\rho(\mathfrak{v})}$ which is a weakest policy in $\Pi_{\lambda_\rho(P)}$. \\

\noindent \textbf{Concluding remarks: } Here we argued software minds are a flawed concept, symptomatic of ``computational dualism''. We argued for an alternative we call pancomputational enactivism, which allows for objective claims regarding behaviour. We used this alternative to propose upper bounds on intelligence, and we anticipate these results may further our understanding of AI safety and AGI. In practical terms this upper bound is unattainable, because we can only build systems with finite vocabularies. Moreover, too large a vocabulary $\mathfrak{v}$ could make $\mathfrak{v}$-tasks intractable. Physical embodiment severely constrains what is possible. Rigorous research has been undertaken regarding the risks of AGI \cite{cohen2022}, but it is based upon computational dualism. Our results suggest AGI will be safer, but more limited, than has been theorised. Our results may also be of use in the pursuit of AGI. Higher utility does mean weaker and more generalise-able policies, which suggests one should optimise for higher utility whilst trying to minimise $\lvert \mathfrak{v} \rvert$. 

\printbibliography

@book{pearl2018,
	author = {Pearl, Judea and Mackenzie, Dana},
	title = {The Book of Why: The New Science of Cause and Effect},
	year = {2018},
	publisher = {Basic Books, Inc.},
	address = {New York},
	edition = {1st}
}

@article{cohen2022,
author = {Cohen, Michael K. and Hutter, Marcus and Osborne, Michael A.},
title = {Advanced artificial agents intervene in the provision of reward},
journal = {AI Magazine},
volume = {43},
number = {3},
pages = {282},
year = {2022}
}

@InCollection{piccinini2021,
	author       =	{Piccinini, Gualtiero and Maley, Corey},
	title        =	{{Computation in Physical Systems}},
	booktitle    =	{The {Stanford} Encyclopedia of Philosophy},
	%editor       =	{Edward N. Zalta},
	howpublished =	{\url{https://plato.stanford.edu/archives/sum2021/entries/computation-physicalsystems/}},
	%address = {Stanford},
	year         =	{2021},
	edition      =	{{S}um. 21},
	publisher    =	{Stanford University}
}

@article{chaitin2006,
    author = {Chaitin, G.},
    year = {2006},
    pages = {74--81},
    title = {The Limits of Reason},
    volume = {294},
    number = {3},
    journal = {Sci. Am.}
}

@InProceedings{bennett2022b,
author="Bennett, Michael Timothy",
title="Compression, The Fermi Paradox and Artificial Super-Intelligence",
booktitle="Artificial General Intelligence",
year="2022",
pages="41--44"
}

@article{bennett2022c,
    author = "Michael Timothy Bennett",
    title = "{Computable Artificial General Intelligence}",
    year = "2022",
    publisher={Under Review}
}

@InProceedings{bennett2024b,
    author = "Michael Timothy Bennett",
    title = "Is Complexity an Illusion?",
    booktitle="Artificial General Intelligence",
    year="2024",
    publisher="Springer",
    address="",
    editor=""
}

@misc{bennett2024c,
    author = "Michael Timothy Bennett",
    title = "Multiscale Causal Learning",
    year = "2024",
    note={Under review}
}

@misc{hinton2022,
      title={The Forward-Forward Algorithm: Some Preliminary Investigations}, 
      author={Geoffrey Hinton},
      year={2022},
      eprint={2212.13345},
      archivePrefix={arXiv},
      primaryClass={cs.LG}
}

@article{clark1998,
	journal = {Analysis},
	volume = {58},
	year = {1998},
	number = {1},
	title = {The Extended Mind},
	pages = {7--19},
	%doi = {10.1093/analys/58.1.7},
	publisher = {Oxford University Press},
	author = {Andy Clark and David J. Chalmers}
}

@article{ward2017,
author = {Ward, Dave and Silverman, David and Villalobos, Mario},
year = {2017},
month = {04},
pages = {},
title = {Introduction: The Varieties of Enactivism},
volume = {36},
journal = {Topoi}%,
%doi = {10.1007/s11245-017-9484-6}
}

@article{legg2007,
	pages = {391--444},
	%doi = {10.1007/s11023-007-9079-x},
	title = {Universal Intelligence: A Definition of Machine Intelligence},
	author = {Shane Legg and Marcus Hutter},
	volume = {17},
	journal = {Minds and Machines},
	publisher = {Springer},
	number = {4},
	year = {2007}
}

@misc{chollet2019,
  %doi = {10.48550/ARXIV.1911.01547},
  %url = {arxiv.org/abs/1911.01547},
  author = {Chollet, François},
  keywords = {Artificial Intelligence (cs.AI), FOS: Computer and information sciences, FOS: Computer and information sciences},
  title = {On the Measure of Intelligence},
  %publisher = {arXiv},
  year = {2019},
  %copyright = {arXiv.org perpetual, non-exclusive license}
}

@book{vitanyi2008,
	publisher = {Springer},
	year = {2008},
	author = {Ming Li and Paul M. B. Vit\'{a}nyi},
	title = {An Introduction to Kolmogorov Complexity and its Applications (Third Edition)},
	address={New York}
}

@article{solomonoff1978,
  author={Ray Solomonoff},
  journal={IEEE TIT}, 
  title={Complexity-based induction systems: Comparisons and convergence theorems}, 
  year={1978},
  volume={24},
  number={4},
  pages={422–432}
}

@InProceedings{bennett2023b,
  author="Bennett, Michael Timothy",
editor="",
title="The Optimal Choice of Hypothesis Is the Weakest, Not the Shortest",
booktitle="Artificial General Intelligence",
year="2023",
publisher="Springer",
address="",
pages="42--51"
}

@InProceedings{bennett2023c,
author="Bennett, Michael Timothy",
editor="",
title="Emergent Causality and the Foundation of Consciousness",
booktitle="Artificial General Intelligence",
year="2023",
publisher="Springer",
address="",
pages="52--61"
}

@InProceedings{bennett2023d,
author="Bennett, Michael Timothy",
editor="",
title="On the Computation of Meaning, Language Models and Incomprehensible Horrors",
booktitle="Artificial General Intelligence",
year="2023",
publisher="Springer",
address="",
pages="32--41"
}

@article{rissanen1978,
  title={Modeling By Shortest Data Description*},
  author={Jorma Rissanen},
  journal={Autom.},
  year={1978},
  volume={14},
  pages={465-471}
}

@book{sober2015, place={Cambridge}, title={Ockham's Razors: A User's Manual}, 
%DOI={10.1017/CBO9781107705937}, 
publisher={Cambridge Uni. Press}, author={Sober, Elliott}, year={2015}}

@article{dreyfus2007,
author = { Hubert L.   Dreyfus },
title = {Why Heideggerian AI Failed and How Fixing it Would Require Making it More Heideggerian},
journal = {Phil. Psych.},
volume = {20},
number = {2},
pages = {247},
year  = {2007},
publisher = {Routledge}%,
%doi = {10.1080/09515080701239510},
%URL = {https://doi.org/10.1080/09515080701239510},
%eprint = {https://doi.org/10.1080/09515080701239510}
}

@InCollection{heidegger2020,
	author       =	{Wheeler, Michael},
	title        =	{{Martin Heidegger}},
	booktitle    =	{The {Stanford} Encyclopedia of Philosophy},
	%editor       =	{Edward N. Zalta},
	howpublished =	{\url{https://plato.stanford.edu/archives/fall2020/entries/heidegger/}},
	year         =	{2020},
	edition      =	{{F}all 2020},
	publisher    =	{Stanford University}
}

@InProceedings{bennett2022a,
    author="Bennett, Michael Timothy",
    %editor="Goertzel, B. and Ikl{\'e}, M. and Potapov, A.",
    title="Symbol Emergence and the Solutions to Any Task",
    booktitle="Artificial General Intelligence",
    year="2022",
    pages="30--40",
    %publisher="Springer",
    %address="Cham",
    %isbn="978-3-030-93758-4"
}

@book{hutter2010,
    author = {Hutter, Marcus},
    title = {Universal Artificial Intelligence: Sequential Decisions Based on Algorithmic Probability},
    year = {2010},
    %isbn = {3642060528},
    publisher = {Springer-Verlag},
    address = {Berlin, Heidelberg}
}

@article{leike2015,
  author={Leike, J. and Hutter, M.},
  journal={Proceedings of The 28th Conference on Learning Theory, in Proceedings of Machine Learning Research}, 
  title={Bad Universal Priors and Notions of Optimality},
  year={2015},
  pages={1244-1259}
}

@book{thompson2007,
	author = {Evan Thompson},
	title = {Mind in Life: Biology, Phenomenology, and the Sciences of Mind},
	publisher = {Harvard University Press},
	address = {Cambridge MA},
	year = {2007}
}

@article{kolmogorov_1963,
  author={Kolmogorov, A.N.},
  journal={Sankhya: The Indian Journal of Statistics}, 
  title={On tables of random numbers},
  year={1963},
  volume={A},
  pages={369-376}
}

@article{artinger2022,
    Author = {Artinger, Florian M. and Gigerenzer, Gerd and Jacobs, Perke},
    Title = {Satisficing: Integrating Two Traditions},
    Journal = {Journal of Economic Literature},
    Volume = {60},
    Number = {2},
    Year = {2022},
    Pages = {598-635}%,
    %DOI = {10.1257/jel.20201396},
    %URL = {https://www.aeaweb.org/articles?id=10.1257/jel.20201396}
}

@book{kim2011,
	address = {New York},
	author = {Jaegwon Kim},
    editor = {},
    edition = {3rd ed.},
	publisher = {Routledge},
	title = {Philosophy of Mind},
	year = {2011}
}

@incollection{putnam1967,
	author = {Hilary Putnam},
	booktitle = {Art, mind, and religion},
	pages = {37--48},
	publisher = {Uni. of Pittsburgh Press},
	title = {Psychological Predicates},
	year = {1967}
}

@book{goertzel2006,
  title={The Hidden Pattern: A Patternist Philosophy of Mind},
  author={Goertzel, B.},
  year={2006},
  publisher={BrownWalker Press}
}

@article{friston2013,
author = {Friston, Karl },
title = {Life as we know it},
journal = {The Royal Society Interface},
year = {2013},
%volume = {10},
%number = {86},
    abstract = { This paper presents a heuristic proof (and simulations of a primordial soup) suggesting that life—or biological self-organization—is an inevitable and emergent property of any (ergodic) random dynamical system that possesses a Markov blanket. This conclusion is based on the following arguments: if the coupling among an ensemble of dynamical systems is mediated by short-range forces, then the states of remote systems must be conditionally independent. These independencies induce a Markov blanket that separates internal and external states in a statistical sense. The existence of a Markov blanket means that internal states will appear to minimize a free energy functional of the states of their Markov blanket. Crucially, this is the same quantity that is optimized in Bayesian inference. Therefore, the internal states (and their blanket) will appear to engage in active Bayesian inference. In other words, they will appear to model—and act on—their world to preserve their functional and structural integrity, leading to homoeostasis and a simple form of autopoiesis. }
}
\end{document}